%% file: scaled_bfp_arxiv.tex
\documentclass[journal]{IEEEtran}
\usepackage{amsmath,graphicx,bm,amssymb,amsthm,enumerate,epsfig,psfrag,cases,mathtools}
\usepackage[shortlabels]{enumitem}
\include{ilyas_commands}

\usepackage{fontenc,multicol,bigints,setspace}
\usepackage{caption}
\usepackage{subcaption}
\usepackage{inputenc}
\usepackage[square,sort,compress,comma,numbers]{natbib}
\usepackage{epstopdf,pst-node}
\usepackage{mathtools}
\DeclarePairedDelimiter{\ceil}{\lceil}{\rceil}

\begin{document}
\title{Block Format Error Bounds \\ and Optimal Block Size Selection}

\author{
    \IEEEauthorblockN{Ilya Soloveychik, Ilya Lyubomirsky, Xin Wang and Sudeep Bhoja} \\
    \IEEEauthorblockA{\normalsize d-Matrix}
}

\maketitle

\begin{abstract}
The amounts of data that need to be transmitted, processed, and stored by the modern deep neural networks have reached truly enormous volumes in the last few years calling for the invention of new paradigms both in hardware and software development. One of the most promising and rapidly advancing frontiers here is the creation of new numerical formats. In this work we focus on the family of block floating point numerical formats due to their combination of wide dynamic range, numerical accuracy, and efficient hardware implementation of inner products using simple integer arithmetic. These formats are characterized by a block of mantissas with a shared scale factor. The basic Block Floating Point (BFP) format quantizes the block scales into the nearest powers of two on the right. Its simple modification - Scaled BFP (SBFP) - stores the same scales in full precision and thus allows higher accuracy. In this paper, we study the statistical behavior of both these formats rigorously. We develop asymptotic bounds on the inner product error in SBFP- and BFP-quantized normally distributed vectors. Next, we refine those asymptotic results to finite dimensional settings and derive high-dimensional tight bounds for the same errors. Based on the obtained results we introduce a performance measure assessing accuracy of any block format. This measure allows us to determine the optimal parameters, such as the block size, yielding highest accuracy. In particular, we show that if the precision of the BFP format is fixed at $4$ bits, the optimal block size becomes $64$. All theoretical derivations are supported by numerical experiments and studies on the weights of publicly available pretrained neural networks.
\end{abstract}


\section{Introduction}
\label{sec:intro}
In the recent years, online services addressing requests of billions of customers worldwide have become ubiquitous. Such customization is usually provided by Deep Neural Networks (DNN) deployed at data centers and equipped with unprecedented computational capabilities. The recently discovered attention paradigm \cite{vaswani2017attention} has led to the explosion in the variety and complexity of the modern high-performance DNNs based on the transformer architecture \cite{vaswani2017attention, lin2021survey, kalyan2021ammus}. Transformers have achieved tremendous success in all Natural Language Processing (NLP) tasks due to their ability to learn universal language representations from large volumes of unlabeled text data and then transfer this knowledge to downstream tasks \cite{devlin2018bert, radford2019language}. Importantly, the most common operation in transformers consuming the majority of the compute and power resources is matrix multiplication \cite{wang2019benchmarking, srinivas2021bottleneck}. Creating the infrastructure that will make the multiplication of the involved large matrices faster and cheaper without impacting the product accuracy has become the central focus of the leading researchers and companies working in this field \cite{elster2022nvidia, pang2022ai, huang2022hardware}.

Enabling faster inference and more reliable training of large DNNs is usually limited by the arithmetic bandwidth of the underlying hardware \cite{wang2019benchmarking, srinivas2021bottleneck, pang2022ai, huang2022hardware}. Quite often designers and operators of transformer-based networks exploit Graphics Processing Unit (GPUs) as the workhorse because they offer higher arithmetic density per silicon area than Central Processing Unit (CPUs) \cite{wang2019benchmarking, srinivas2021bottleneck}. Inference and training tasks in all these cases are usually processed through full precision floating-point units. However, the computational load required by modern transformers has reached such enormous volumes that traditional GPUs cannot fully meet the growing demand, pushing both accelerators and high performance GPUs towards narrow arithmetic. As a consequence, unmatched research efforts have been applied by the engineering community to replace narrow floating-point with even denser fixed-point representations \cite{zadeh2020gobo, zafrir2019q8bert, shen2020q, zhang2020ternarybert}. Despite the excellent gains in both speed and computational density achieved by fixed-point arithmetic, training using it or even half-precision floating-point arithmetic has not provided clear evidence in its favor due to the limited dynamic range inherent in such formats \cite{micikevicius2017mixed}.

Block Floating Point (BFP) numerical formats have received renewed interest recently for DNNs inference applications due to their combination of wide dynamic range, numerical accuracy, and efficient hardware implementation of inner products using simple integer arithmetic \cite{darvish2020pushing, lyubomirsky2022clock}. BFP formats are characterized by a block of mantissas with a shared scale factor. The simplest implementation has the scale factor as a power of two, the so-called exponent, in which case the inner product between two blocks involves multiplying the integer mantissas and adding the two block exponents.

Another useful version of the BFP format involves using a floating point scale factor. The latter provides more numerical accuracy at the expense of slightly increased complexity for inner product operation since we need a floating point multiplication of the scale factors \cite{dai2021vs}. In this paper, we consider both formats calling the former simply BFP and the latter SBFP, where the S stands for floating point scale factor. Also, for simplicity and without much impact on the result, in this study we treat the block scales of SBFP as infinite precision numbers. 

The contribution of this article is three-fold. First, we develop asymptotic bounds on the inner product error in SBFP- and BFP-quantized normally distributed vectors. Second, we refine those asymptotic results to finite dimensional settings and derive high-dimensional tight bounds for the same errors. Third, we utilize the obtained results to introduce a performance measure assessing the typical accuracy of any block format. Such a framework enables us to determine the optimal values of the block format parameters providing the highest possible accuracy. An example of such parameter selection is the block size tuning, which we study in detail. In particular, we show that if the precision of the BFP format is fixed at $4$ bits, the optimal block size becomes $n=64$. All our theoretical derivations are supported by numerical simulations and studies on the weights of publicly available pretrained neural networks. To the best of our knowledge this article is the first work shedding light on the statistical properties of block formats through rigorous studies.

The rest of the text is organized as follows. First, we define the formats and state the problem of bounding the inner product quantization error in Sections \ref{sec:quant} and \ref{sec:stat_model}, respectively. In Section \ref{sec:asymp} we formulate the asymptotic bounds for large block sizes which provide much insight into the behavior of the scalar product quantization error. Section \ref{sec:highdim} then refines the asymptotic bounds into high-dimensional bounds which better suit the needs of small block accuracy analysis and optimal parameter selection in practice. In Section \ref{sec:optimal_size} we introduce a new framework allowing assessment of any block formats with respect to the optimal behavior of SBFP and show how such a methodology can be used to determine the optimal block size for any block format and dataset. Section \ref{sec:num_res} provides extensive numerical simulations using synthetic data as well as publicly available pretrained neural networks supporting our theoretical findings and conclusions. The technical details of the proofs can be found in the Appendix.

\section{Block Quantization}
\label{sec:quant}
Recent extensive studies of popular DNNs based on the transformer mechanism have shown that their trained weights often have comparable amplitudes \cite{darvish2020pushing, lyubomirsky2022clock}. This observation naturally leads to the following simple compression algorithm. Given a list of weights, we can store some common scaling factor in high precision and a set of small precision integers for each element of the list. This idea is formalized by the family of block formats studied in this work. 

\begin{definition}[Block Format]
A block format is a triple of a block size $n \in \mathbb{N}$, mantissa precision $p \in \mathbb{N}$, and a numerical format for the scaling factor $S$. Given such a triple, all the elements $\{M_i\}_{i=1}^n$ of a block are stored as integers aka mantissas in the range of $\[-(2^{p-1}-1),\, 2^{p-1}-1\]$, and their values are computed as
\begin{equation}
\{S\cdot M_1,\dots,S\cdot M_n\}.
\end{equation}
\end{definition}

The block size $n$ and precision $p$ usually play the role of degrees of freedom or format parameters that can be altered. Therefore, a block format is determined by the scale factor numerical format. 
%In this paper we focus on two block formats determined by the scale format. For simplicity of exposition, we start with the SBFP format storing the scale with infinite accuracy.

\subsection{SBFP Quantization}
\label{sec:quant_scheme}
Let us fix the block size $n$ and mantissa precision $p$ for a moment. Assume we are given $n$ real numbers $X_1,\dots,X_n \in \mathbb{R}$ that we want to store in a block. Denote their maximal absolute value by
\begin{equation}
Y = \max_{i=1}^n |X_i|,
\end{equation}
%\begin{equation}
%Z_i = \frac{X_i}{Y} \in [-1,\, 1],\;\; i=1,\dots,n.
%\end{equation}
and set
\begin{equation}
\alpha = 2^{p-1}-1.
\end{equation}
Let us rescale the $X_i$ values as
\begin{equation}
Z_i = \alpha\frac{X_i}{Y} \in [-\alpha,\, \alpha],\;\; i=1,\dots,n.
\end{equation}
Next we round every $Z_i$ to the closest integer and these will be their corresponding mantissas, $\mathcal{I}\[Z_i\]$. The Scaled Block Floating Point (SBFP) representation is now characterized by the scale $Y$ which is stored in full precision and $n$ integers $\mathcal{I}\[Z_i\],\; i=1,\dots,n$ stored using $p$ bits each.

%\begin{figure}
%\centering
%\includegraphics[width=0.7\linewidth]{pics/format_layout.png}
%\caption{Schematic layouts of SBFP and BFP layouts for specific clock siz.}
%\label{fig:mem_layout}
%\end{figure}

%Denote the block elements by $X_i^{(1)}$ and their maximal absolute value by
%\begin{equation}
%Y^{(1)} = \max_{i=1}^n |X_i^{(1)}|.
%\end{equation}
%The distribution of $Y^{(1)}$ is given by Lemma \ref{lem:gumbel_dist_y} below. To properly rescale the mantissas that will be rounded to integers, we divide all $X_i^{(1)}$ by $Y^{(1)}$ and multiply by
%\begin{equation}
%\alpha_1 = 2^{p_1-1}-1
%\end{equation}
%to get
%\begin{equation}
%Z_i^{(1)} = \alpha_1\frac{X_i^{(1)}}{Y^{(1)}} \in [-(2^{p_1-1}-1),\, 2^{p_1-1}-1].
%\end{equation}
%Each $Z_i^{(1)}$ is now rounded to the nearest integer. Denote by
%\begin{equation}
%\Delta Z_i^{(1)} = Z_i^{(1)} - \mathcal{I}\[Z_i^{(1)}\]
%\end{equation}
%the respective absolute error, where $\mathcal{I}\[\cdot\]$ is the rounding to the nearest integer operator. Note also that for some $i_*^{(1)}$, $Z_{i_*}^{(1)}$ is equal to $\alpha_1$ or $-\alpha_1$ and is therefore nonrandom. Already for moderate values of $p_1$, we can assume without loss of accuracy that $\Delta Z_i^{(1)},\; i \neq i_*^{(1)}$ are uniformly distributed within $\[-\frac{1}{2},\frac{1}{2}\]$.

%\begin{figure}
%\centering
%\includegraphics[width=0.7\linewidth]{pics/format_layout.png}
%\caption{Schematic layouts of SBFP and BFP layouts for specific clock siz.}
%\label{fig:mem_layout}
%\end{figure}

\subsection{BFP Quantization}
\label{sec:bfp_quant}
For the plain Block Floating Point (BFP) quantization format, the way we memorize the scale changes. Here, instead of storing the maximal absolute value we keep its nearest power of $2$ from above,
\begin{equation}
\overline{\frac{Y}{\alpha}} = 2^{\ceil*{\log_2 \frac{Y}{\alpha}}},
\end{equation}
where in the right-hand side we used the standard notation for the ceiling function. As a consequence, quite often the new scalar will be significantly larger than the optimal choice which might have detrimental effect on the resulting accuracy, as we shall see later.

\section{Statistical Model}
\label{sec:stat_model}
As mentioned earlier, one of the central problems in the modern hardware development focusing at large-scale DNNs is designing processing units capable of accurate matrix multiplication under heavy memory, compute, and power limitations \cite{wang2019benchmarking, srinivas2021bottleneck,elster2022nvidia, pang2022ai, huang2022hardware}. An inherent part of this endeavor is examining the accuracy of the product as a function of the numerical format used to quantize the multipliers. In this work, we focus on the elementary building block of matrix multiplication - the scalar product of two vectors. To study the properties of the SBFP and BFP quantized inner products systematically and formulate universally meaningful claims, we have to exploit well-specified populations of weights that would be both useful in practice and amenable to rigorous derivations. We chose to use normally distributed data. First, normal populations well approximate real distributions of weights in neural networks because the latter are usually initialized normally at random and on average shift only slightly during the training process \cite{franchi2020tradi}. Second, normal distributions are easy to deal with from both theoretical and practical points of view. Namely, quite often theoretical results involving normal populations allow solution by quadrature yielding meaningful expressions that provide important insights into the nature of the observed phenomena, like in our case. In addition, normal variables are easy to simulate due a huge number of existing algorithms generating normal pseudo-random numbers of good quality.

% Let us start with the SBFP quantization and d

Assume we have two vectors $\left\{X_i^{(j)}\right\}_{i=1}^n,\; j=1,2$ of $n$ random variables which are i.i.d. (independent and identically distributed) normal of mean zero and variance $\sigma^2$ each. After quantizing them into a block format using the scheme described in Section \ref{sec:quant_scheme} above, denote the absolute quantization error of each element by
\begin{equation}
\Delta Z_i^{(j)} = Z_i^{(j)} - \mathcal{I}\[Z_i^{(j)}\],\;\; i=1,\dots,n,\;j=1,2.
\end{equation}
Note that for each $j$ there exists an index $i_*^{(j)}$, for which $Z_{i_*}^{(j)}$ is equal to $\alpha_j$ or $-\alpha_j$ and is therefore nonrandom. Note also that already for moderate values of precision $p_j$, we can assume without loss of accuracy that $\Delta Z_i^{(j)},\; i \neq i_*^{(j)}$ are uniformly distributed within $\[-\frac{1}{2},\frac{1}{2}\]$.

For the SBFP format study, the quantity we shall be focusing on below is the absolute error of the scalar product of two blocks,
\begin{equation}
\label{eq:sc_error}
\Delta E_s = \sum_{i}X_i^{(1)}X_i^{(2)} - \frac{Y^{(1)}}{\alpha_1}\frac{Y^{(2)}}{\alpha_2}\sum_i \mathcal{I}\[Z_i^{(1)}\]\mathcal{I}\[Z_i^{(2)}\],
\end{equation}
where the subscript $s$ stands for SBFP quantization. An analogous error for BFP will be denoted as $\Delta E_b$. In this paper we will be focusing on the study of the statistical properties of these errors as a function of block format parameters, such as precision $p$ and block size $n$.

%\begin{prop}[SBFP Quantization]
%\label{prop:sbfp_bound}
%For any $t\geqslant 0$, the probability of large deviations of the scalar product error $\Delta E$ defined in (\ref{eq:sc_error}) can be bounded as
%\begin{multline}
%\label{eq:sbfp_bound}
%\mathbb{P}\[\left|\Delta E\right| \geqslant t\] \leqslant 2\exp\(-\frac{t^2}{\frac{(n-1)\sigma^4}{6}\[\alpha_1^{-2} + \alpha_2^{-2} + (\alpha_1\alpha_1)^{-2}\ln(2n)/6\]\ln\(\frac{4n^2}{2\pi\ln(2n^2/\pi)}\)}\), \\ n \to \infty.
%\end{multline}
%Consequently, the variance of the error can be bounded as
%\begin{equation*}
%\var{\Delta E} \leqslant \frac{(n-1)\sigma^4}{12}\[\alpha_1^{-2} + \alpha_2^{-2} + (\alpha_1\alpha_1)^{-2}\ln(2n)/6\]\ln\(\frac{4n^2}{2\pi\ln(2n^2/\pi)}\),\;\; n \to \infty.
%\end{equation*}
%\end{prop}

%As an immediate corollary we get the following claim.
%\begin{corollary}[SBFP Quantization]
%For $1 \ll \ln\, n \ll \max_j 2^{p_j}$, the variance of the error of the inner product of SBFP-quantized vectors can be bounded by
%\begin{equation}
%\var{\Delta E} \leqslant \frac{\sigma^4}{12}\cm{\[\alpha_1^{-2} + \alpha_2^{-2}\]}\cmr{n\,\ln\(\frac{4n^2}{2\pi\ln(2n^2/\pi)}\)}.
%\end{equation}
%\end{corollary}

\section{Asymptotic Bounds}
\label{sec:asymp}
We start by deriving the limiting bounds on the approximation quality of SBFP and BFP formats when the block size $n$ tends to infinity. Such results will provide a number of useful insights into the quantization properties of the formats at hand. Later we refine them for finite block sizes. To rigorously formulate the results will use the following auxiliary definition.
\begin{definition}[\cite{vershynin2018high}]
\label{def:subgauss}
We shall say that a real centered random variable $W$ is sub-Gaussian with variance proxy $\sigma^2>0$ if its moment generating function can be bounded as
\begin{equation}
\mathbb{E}\[e^{sW}\] \leqslant e^{\frac{\sigma^2s^2}{2}},\;\; \forall s \in \mathbb{R}.
\end{equation}
\end{definition}
As follows from the name, the statistical properties of sub-Gaussian random variables mimic those of Gaussian ones. Lemma \ref{lem:subgaus_prop} from the Appendix lists a few such properties that will be useful in developing the bounds at question. For a more complete review of sub-Gaussianity we refer the reader to \cite{vershynin2018high} and the references cited therein.

\subsection{SBFP Quantization}
\begin{prop}[Asymptotic SBFP Bound]
\label{prop:sbfp_bound_asymp}
For $1 \ll \ln\, n \ll \max_j 2^{p_j}$, the variance of the error of the inner product of SBFP-quantized vectors can be bounded by
\begin{equation}
\label{eq:var_sbfp_loose_bound}
\var{\Delta E_s} \leqslant \frac{\sigma^4}{8}\[\frac{1}{2^{2(p_1-1)}} + \frac{1}{2^{2(p_2-1)}}\]n\,\ln\(\frac{4n^2}{2\pi\ln(2n^2/\pi)}\),
\end{equation}
moreover $\Delta E_s$ is a symmetric sub-Gaussian random variable with variance proxy $2\var{\Delta E_s}$ and the tail bound
\begin{equation}
\label{eq:sbfp_bound}
\mathbb{P}\[\left|\Delta E_s\right| \geqslant t\] \leqslant 2\exp\(-\frac{t^2}{4\var{\Delta E_s}}\).
\end{equation}
\end{prop}
\begin{proof}
The proof can be found in the Appendix.
\end{proof}

Following the same line as above, we can get a similar asymptotic result for the BFP format as well. For simplicity of notation we assume $p_1=p_2$.
\begin{prop}[Asymptotic BFP Bound]
\label{prop:bfp_bound_asymp}
If $p_1=p_2=p$ and $1 \ll \ln\, n \ll 2^p$, the variance of the error of the inner product of BFP-quantized vectors can be bounded by
\begin{equation}
\label{eq:var_bfp_loose_bound}
\var{\Delta E_b} \leqslant \frac{\sigma^2}{4} n 2^{2\ceil*{\log_2 \(\frac{\sigma}{2^{p-1}}\) + \frac{1}{2}\log_2 \ln\(\frac{4n^2}{2\pi\ln(2n^2/\pi)}\)}},
\end{equation}
moreover $\Delta E_b$ is a symmetric sub-Gaussian random variable with variance proxy $2\var{\Delta E_b}$ and the tail bound
\begin{equation}
\label{eq:sbfp_bound}
\mathbb{P}\[\left|\Delta E_b\right| \geqslant t\] \leqslant 2\exp\(-\frac{t^2}{4\var{\Delta E_b}}\).
\end{equation}
\end{prop}
\begin{proof}
The proof can be found in the Appendix.
\end{proof}

A few notes are in place here. First, we see that the behavior of the bounds as functions of the mantissas precision are similar, namely both decrease with $p_j$ exponentially. This observation is confirmed by extensive numerical experiments in Section \ref{sec:num_res} below. Second, the dependency of the variance on the block size $n$ can be broken into two multiplicative factors: 1) the linear dependency on $n$ and 2) the logarithmic dependency on $n$. The linear dependency is due to the proportional growth of the variance of the error with the length of the vectors being multiplied. Indeed, summing up more i.i.d. variables naturally leads to the linear grows of the variance. The logarithmic multiplier is more surprising. In fact, it comes from the distribution of the maximum of Gaussian random variables as explained in the Appendix section. Here we only mention that given $n$ i.i.d. standard normal random variables, the expected maximum of the absolute values of these variables grows approximately as $\sqrt{2\log{n}}$. Since the variance is proportional to the square of the expected block scales - as explained in (\ref{eq:sc_error}) above and Lemma \ref{lem:fixed_y_bound} below - its growth with $n$ becomes logarithmic. All these observations are also verified by the numerical experiments of Section \ref{sec:num_asymp}. Note also that sub-Gaussianity of the errors imply their tight concentration around the mean. This implies that the obtained bounds, although asymptotic, are quite tight.

It is important to emphasize the presence of the ceiling function in the expression (\ref{eq:var_bfp_loose_bound}) of the BFP bound. This leads to jumps in the values of the bound at those block sizes, at which the expected values of the scales $\frac{Y_n}{\alpha}$ pass over integer powers of two. A more detailed discussion supported by numerical simulations can be found in Section \ref{sec:num_asymp} below. 

\section{High-dimensional Bounds}
\label{sec:highdim}
Numerical simulations of Section \ref{sec:num_asymp} show that for finite block sizes $n$ the bound obtained in Proposition \ref{prop:bfp_bound_asymp} becomes quite loose. In this section, we provide finite-dimensional improvements of the variance bounds from Propositions \ref{prop:sbfp_bound_asymp} and \ref{prop:bfp_bound_asymp}. Such improvements however come at a price. Here we already cannot solve the result by quadrature and are forced to use numerical integration for our simulations in Section \ref{sec:num_res}.

\begin{lemma}
\label{lem:fin_dim_distr}
Let $\Phi$ and $\phi$ denote the cdf and pdf of the standard normal distribution, respectively, and for some $\sigma^2>0$ $X_i \sim \mathcal{N}(0, \sigma^2),\; i=1,\dots,n$ be independent. Then the distribution of the random variable
\begin{equation}
Y = \max_{i=1}^n |X_i|
\end{equation}
reads as
\begin{equation}
\label{eq:density_phi}
f_n(y) = f_{Y,n}(y) = \frac{2n}{\sigma}\phi\(\frac{y}{\sigma}\)\[2\Phi\(\frac{y}{\sigma}\) - 1\]^{n-1}.
\end{equation}
\end{lemma}
\begin{proof}
For independence,
\begin{equation}
\mathbb{P}\[Y\leqslant y\] = \prod_i \mathbb{P}\[|X_i| \leqslant y\] = \prod_i \mathbb{P}\[ -y \leqslant X_i \leqslant y\].
\end{equation}
Since each $X_i$ is normal and $y$ is obviously non-negative, we obtain
\begin{equation}
\mathbb{P}\[Y\leqslant y\] = \[2\Phi\(\frac{y}{\sigma}\) - 1\]^n.
\end{equation}
Now take the first derivative to conclude the proof. 
\end{proof}

\begin{prop}[High-dimensional variance of SBFP Quantization]
\label{prop:hd_sbfp_bound}
If $1 \ll \ln\, n \ll \max_j 2^{p_j}$, the variance of the error of the inner product of SBFP-quantized vectors can be bounded by
\begin{align}
\label{eq:var_sbfp_hd_bound}
&\var{\Delta E_s} \nonumber \\
&\leqslant \frac{n\sigma^4}{8} \int \[\(\frac{y_1}{\alpha_1}\)^2 + \(\frac{y_2}{\alpha_2}\)^2\] f_n(y_1)f_n(y_2)dy_1dy_2.
\end{align}
\end{prop}
\begin{proof}
The proof can be found in the Appendix.
\end{proof}

A natural extension of this result to the BFP format reads as follows.

\begin{prop}[High-dimensional variance of BFP Quantization]
\label{prop:hd_bfp_bound}
If $p_1=p_2=p$ and $1 \ll \ln\, n \ll 2^p$, the variance of the error of the inner product of BFP-quantized vectors can be bounded by
\begin{align}
\label{eq:var_bfp_hd_bound}
&\var{\Delta E_b} \\
&\leqslant \frac{n\sigma^2}{8} \int \[2^{2\ceil*{\log_2 \frac{\sigma y_1}{\alpha_1}}} + 2^{2\ceil*{\log_2 \frac{\sigma y_2}{\alpha_2}}}\] f_n(y_1)f_n(y_2)dy_1dy_2. \nonumber
\end{align}
\end{prop}
\begin{proof}
The proof can be found in the Appendix.
\end{proof}

The bounds of equations (\ref{eq:var_sbfp_hd_bound}) and (\ref{eq:var_bfp_hd_bound}) are already harder to analyze through pure observation than (\ref{eq:var_sbfp_loose_bound}) and (\ref{eq:var_bfp_loose_bound}) due to the involved integrals which cannot be solved by quadrature. To get a better understanding of the performance characteristics of the high-dimensional results of this section, we resort to numerical simulations in Section \ref{sec:num_hd_b}.

\section{Block Format Performance}
\label{sec:optimal_size}
Block quantization techniques are becoming extremely popular in all modern machine learning applications requiring storage, communication, and arithmetic processing of tensors of enormous sizes. All block formats store the mantissas of the elements as integers and differ mainly by the format of the scale.

Naturally, SBFP is the most accurate block format since it enables storage of the scale in full precision. As a consequence, all other block formats would usually have larger errors making SBFP format a suitable benchmark for measuring the relative performance quality of any other block format with respect to it. Clearly, the performance of every block format may be different on different datasets. To introduce a universal performance measure of a format, we need to evaluate it on well-specified universal populations. In this study, we use normal centered populations for such evaluation for the reasons listed in Section \ref{sec:stat_model}.

\subsection{Relative Block Format Accuracy Measure}
Based on the discussion above, we propose the following RElative Block format ACcuracy (REBAC) measure for any block format F,
\begin{equation}
\rho_m(F) = \frac{m_{F}}{m_{SBFP}},
\end{equation}
where $m$ could be any statistic evaluating the quantization quality of a block format. For instance, in this article we suggest using the variance of the scalar product quantization error in normal populations. Thus, our REBAC measure reads as
\begin{equation}
\label{eq:rho_var}
\rho_{\text{var}}(F) = \frac{\var{\Delta E_f}}{\var{\Delta E_s}},
\end{equation}
where $\Delta E_f$ is the error induced by quantizing into the F format and the original underlying block values are normally distributed. Next we show the practical importance of the REBAC measure.

\begin{figure}
\centering
\includegraphics[width=1.0\linewidth]{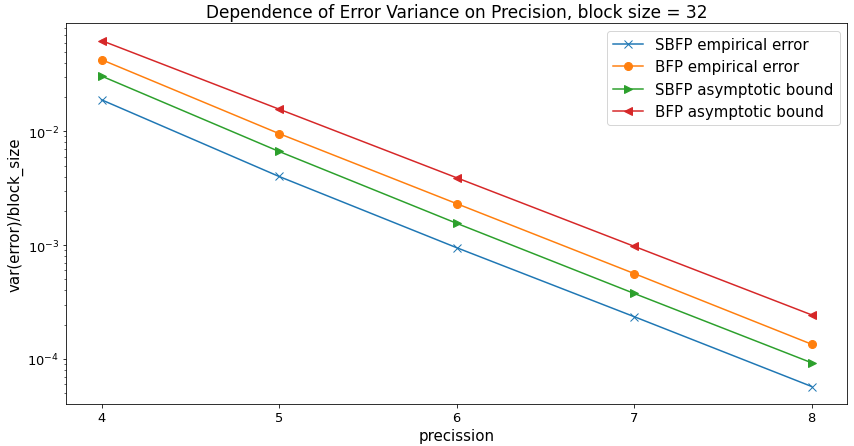}
\caption{Asymptotic bounds as functions of precision.}
\label{fig:var_p_asymp}
\end{figure}

\subsection{Optimal Block Size}
Given a block format, the main problems faced by the practitioners are the choice of the mantissa bit width and the block size. The choice of the mantissa bit width is simplified by an easy scaling rule. As Figure \ref{fig:var_p_asymp} suggests, each extra bit of precision causes approximately $6$dB decrease in the error variance. 

The choice of the block size is somewhat more complicated. Here, the REBAC ratio introduced above provides a natural framework for the optimal selection of the block size. Indeed, if we consider the behavior of the REBAC ratio as a function of the block size, we would often be able to select the block size at which a minimum is attained. Section \ref{sec:opt_choice_num} compares the theoretically optimal choices of the block sizes for various precisions with their empirical values established for the weights of GPT2-XL network \cite{radford2019language}.

\section{Numerical Simulations}
\label{sec:num_res}
In this section, we support our theoretical findings by empirical studies. First, we verify that the bounds obtained in Propositions \ref{prop:sbfp_bound_asymp}, \ref{prop:bfp_bound_asymp}, \ref{prop:hd_sbfp_bound} and \ref{prop:hd_bfp_bound} indeed apply to normally distributed weights. More importantly, we demonstrate the practical benefits of our bounds and conclusions by applying them to the weights of a publicly available pretrained neural network. In this article we use GPT2-XL available at \cite{radford2019language} for this purpose.

% and  and comparison to the behavior of weights in publicly avalaible pretrained netwroks. . The bounds obtained in Propositions \ref{} and \ref{} are e compare the bounds Figures A and B compare the empirical variance of the scalar product error with the theoretically predicted bounds. The former shows the dependence of the error on the precision of the SBFP mantissas, while the letter demonstrates the dependence on the block size.

\subsection{Asymptotic Bounds}
\label{sec:num_asymp}
Let us first examine the asymptotic bounds developed in Section \ref{sec:asymp}. Figure \ref{fig:var_n_asymp} compares both the SBFP and BFP bounds from Propositions \ref{prop:sbfp_bound_asymp} and \ref{prop:bfp_bound_asymp} to the empirical variances of the inner product quantization errors. To make the visual comparison easier here and everywhere below we divide both the errors and the bounds by the block size. A number of observations are in place here. First, we see that the SBFP bound becomes quite tight already around the block size of $n=16$ despite the asymptotic nature of the former. Second, we can observe the jumps of the BFP bound naturally arising at those block sizes where the expected values of the block scales pass over the powers of two. Due to the asymptotic nature of this bound we do not expect it to upper limit the empirical error everywhere but we can already see that it adequately describes the general behavior of the empirical error curve.

%\begin{figure}
%\centering
%\begin{subfigure}{.5\textwidth}
%  \centering
%  \includegraphics[width=1.0\linewidth]{pics/var_error_n_11.png}
%  \caption{Asymptotic bounds as functions of block size.}
%  \label{fig:var_n_asymp}
%\end{subfigure}%
%\begin{subfigure}{.5\textwidth}
%  \centering
%  \includegraphics[width=1.0\linewidth]{pics/var_error_p_11.png}
%  \caption{Asymptotic bounds as functions of precision.}
%  \label{fig:var_p_asymp}
%\end{subfigure}
%\caption{Dependence of the asymptotic error bounds on block size and precision.}
%\label{fig:var_p_n}
%\end{figure}
%
%\begin{figure}
%\centering
%\begin{subfigure}{.5\textwidth}
%  \centering
%  \includegraphics[width=1.0\linewidth]{pics/confid_ribbon_4.png}
%  \caption{Confidence ribbon for $p=4$.}
%  \label{fig:conf_ribbon_4}
%\end{subfigure}%
%\begin{subfigure}{.5\textwidth}
%  \centering
%  \includegraphics[width=1.0\linewidth]{pics/confid_ribbon_468.png}
%  \caption{Comparison of confidence ribbons for $p=4,6,8$.}
%  \label{fig:conf_ribbon_468}
%\end{subfigure}
%\caption{Confidence ribbons $\pm$ standard deviation around the curve $\mathbb{E}\[\frac{Y_n}{\alpha}\]$ as functions of the block size $n$ for $p=4, 5, 6$.}
%\label{fig:var_p_n}
%\end{figure}

To better understand the behavior of the asymptotic bounds in Figure \ref{fig:var_n_asymp} let us examine Figure \ref{fig:conf_ribbon_4} for a moment. It demonstrates the behavior of the factor $\frac{Y_n}{\alpha}$ as a function of the block size $n$ for a fixed value of $\alpha = 2^{p-1}-1$. More specifically, for each $n$ the graph features the expected value of $\frac{Y_n}{\alpha}$ in dark blue and the confidence interval of plus/minus standard deviation around it. Let us concentrate on the points of intersection of the middle dark blue curve with the horizontal lines drawn at the integer powers of two. We see that in the given range of values we have two such intersections, one around $n=8$ and one around $n=1024$. These are exactly the points at which the asymptotic BFP bound jumps in Figure \ref{fig:var_n_asymp}. Indeed, as Figure \ref{fig:conf_ribbon_4} suggests it is these values of the block size at which the expected block scales pass over a power of two. As a consequence, any larger block size would require usage of a larger (next integer) exponent for the corresponding typical block, leading to a significant drop of quantization accuracy right after that value of $n$. This explains the sharp jumps in Figure \ref{fig:var_n_asymp}. Figure \ref{fig:conf_ribbon_468} shows three confidence ribbons for different values of the precision parameter. From it we can infer that for larger values of $p$ and $\alpha$ the corresponding jumps of the asymptotic BFP bounds would shift to the right. This observation is confirmed by Figure \ref{fig:var_n_asymp_bfp_p} illustrating the behavior of the BFP asymptotic error bound for the same set of precision values $p$.

Figure \ref{fig:var_p_asymp} shows that the dependence on the precision $p$ is also very well captured by the bounds. Note that for all formats each additional bit of precision lowers error variance by factor of $4$. Here and below we always set the precisions of the multiplied vectors to be equal $p_1=p_2$.

\begin{figure*}
\centering
\makebox[\textwidth]
{
\begin{subfigure}{.5\textwidth}
  \centering
  \includegraphics[width=1.0\linewidth]{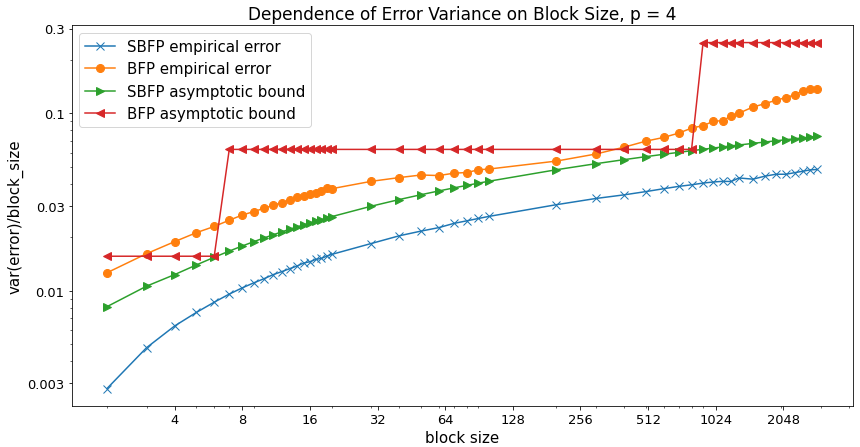}
  \caption{Asymptotic bounds as functions of block size.}
  \label{fig:var_n_asymp}
\end{subfigure}%
\begin{subfigure}{0.5\textwidth}
  \centering
  \includegraphics[width=1.0\linewidth]{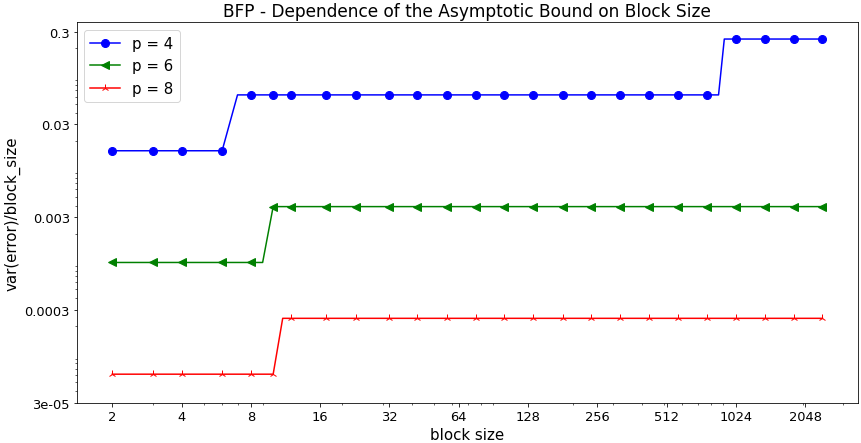}
  \caption{BFP asymptotic bounds as functions of block size.}
  \label{fig:var_n_asymp_bfp_p}
\end{subfigure}%
} \par
\makebox[\textwidth]
{
\begin{subfigure}{.5\textwidth}
  \centering
  \includegraphics[width=1.0\linewidth]{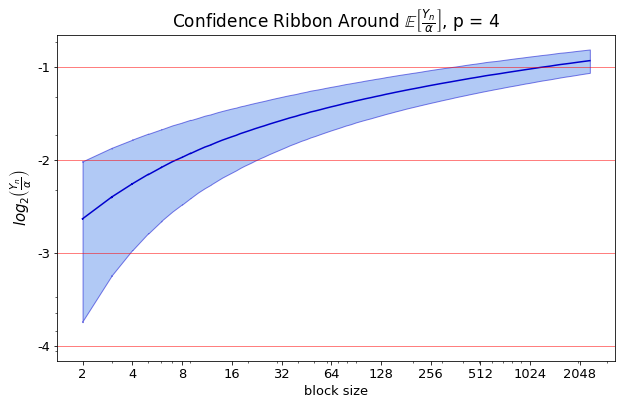}
  \caption{Confidence ribbon for $p=4$.}
  \label{fig:conf_ribbon_4}
\end{subfigure}
\begin{subfigure}{0.5\textwidth}
  \centering
  \includegraphics[width=1.0\linewidth]{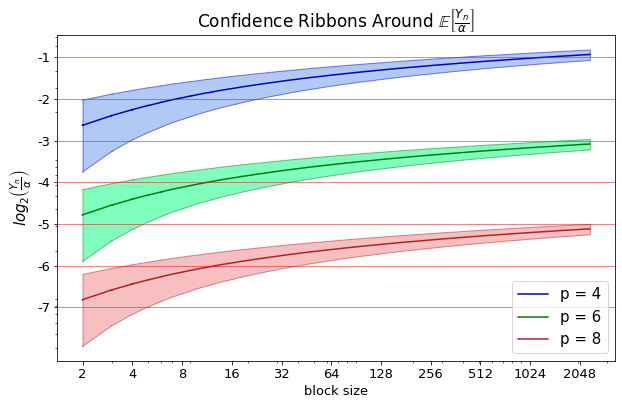}
  \caption{Comparison of confidence ribbons for $p=4,6,8$.}
  \label{fig:conf_ribbon_468}
\end{subfigure}
} \par
\caption{Theoretical error bounds versus the empirical error variances in quantized normal populations and GPT2-XL weights.}
\label{fig:test}
\end{figure*}

\subsection{High-dimensional Bounds}
\label{sec:num_hd_b}
In this section we inspect the improved high-dimensional bounds obtained in Propositions \ref{prop:hd_sbfp_bound} and \ref{prop:hd_bfp_bound}. As discussed in Section \ref{sec:highdim}, here we no longer can get analytical expressions for the bounds but are forced to resort to numerical integration. Figures \ref{fig:hd_sbfp} and \ref{fig:hd_bfp} compare the empirical errors to both asymptotic and high-dimensional bounds for SBFP and BFP formats, respectively. We can see that already for moderate values of $n$ the bounds explain the behavior of the error variance quite well.

%\begin{figure}
%\centering
%\begin{subfigure}{.5\textwidth}
%  \centering
%  \includegraphics[width=1.0\linewidth]{pics/var_error_n_sbfp_4_precise.png}
%  \caption{SBFP}
%  \label{fig:hd_sbfp}
%\end{subfigure}%
%\begin{subfigure}{.5\textwidth}
%  \centering
%  \includegraphics[width=1.0\linewidth]{pics/var_error_n_bfp_4_precise.png}
%  \caption{BFP}
%  \label{fig:hd_bfp}
%\end{subfigure}
%\caption{Asymptotic and high-dimensional error bounds on the variance of the scalar product of SBFP- and BFP-quantized Gaussian vectors versus the empirical errors.}
%\label{fig:test}
%\end{figure}

Comparing Figures \ref{fig:hd_sbfp} and \ref{fig:hd_bfp} we can easily observe different characteristic behavior of the bounding curves. Indeed, the BFP curve changes its curvature and its second derivative clearly passes through zero, while the SBFP curve is always concave. This interesting feature can be easily explained if we recall the asymptotic bounds in Figure \ref{fig:var_n_asymp} reproduced here for convenience. Indeed, the green curve of Figure \ref{fig:hd_bfp} is trying to smoothly approximate its limiting orange bound with jumps. The nature of the bound stays the same - when the bell-shaped distribution of the block scales passes over an integer power of two, we see the change in the curvature of the bounding curve. For large values of $n$, these two curves will converge to the jumpy asymptotic bound. We can see already in Figure \ref{fig:hd_bfp} that the bounds almost coincide for $n > 4096$, while verifying this further would require simulating the empirical error for the values of $n$ near the next jump. Simple calculations imply that the next jump would happen around such large values of $n$ for which it would be impossible to simulate the empirical error on any reasonable computer.

\subsection{GPT2-XL Weights Quantization}
Next we demonstrate the power of the proposed methodology on the weights of the GPT2-XL network pretraied by the authors of \cite{radford2019language}. This network consists of $48$ decoders whose fully connected (FFN) layers are of dimensions $1600 \times 6400$ and $6400 \times 1600$, respectively. We focus on these layers and for the purpose of our inner product studies, we compute the inner products of the different $6400$-dimensional rows and columns of these matrices (in other words we consider the diagonal of the operator product of the two aforementioned matrices).

Figures \ref{fig:gpt2_sbfp} and \ref{fig:gpt2_bfp} compare the quantization errors of our inner products in every layer to the corresponding SBFP and BFP bounds. The quantization of matrices always happens along the larger dimension. We can observe that the bounds work very well up to certain block sizes after which they start diverging. The reason of this divergence for large block sizes is in part simply because for those blocks we get poor averaging of the errors since we have very few blocks of those large sizes.

\subsection{Optimal Block Size}
\label{sec:opt_choice_num}
In this section we explain, how the framework developed in Section \ref{sec:optimal_size} can be used in practice. Figure \ref{fig:rel_m_theory} demonstrates the theoretical behavior of the ratio of the error variances for BFP and SBFP formats - the values of $\rho_{\text{var}}(\text{BFP})$ introduced by equation (\ref{eq:rho_var}). We see that for different precisions the minima are attained at different block sizes. For example, for $p=4$ the optimal block size lays between $n=64$ and $n=128$ while for higher precisions the optimal size grows larger and reaches $n=512$ for $p=8$. We can also see that our theoretical findings are well supported by the behavior of the REBAC measure in GPT2-XL weights, as shown in Figure \ref{fig:rel_m_gpt2}. Here, we averaged the ratio of the variances over all $48$ decoders.

\begin{figure*}
\centering
\makebox[\textwidth]
{
\begin{subfigure}{.5\textwidth}
  \centering
  \includegraphics[width=1.0\linewidth,keepaspectratio]{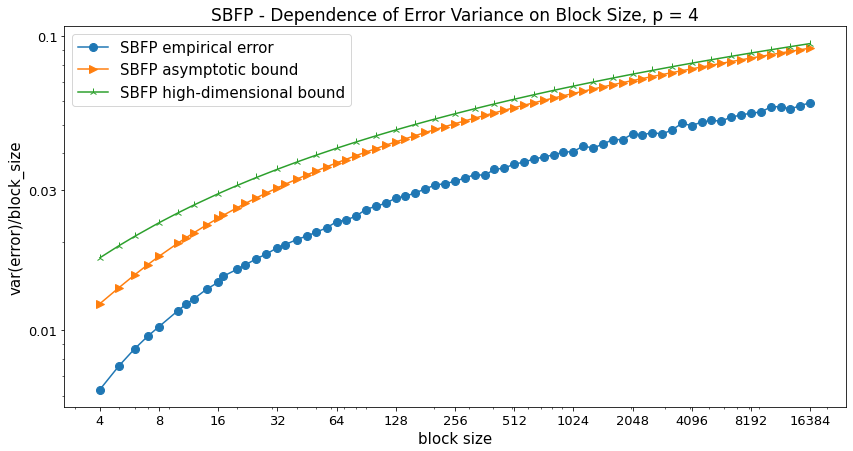}
  \caption{SBFP}
  \label{fig:hd_sbfp}
\end{subfigure}%
\begin{subfigure}{0.5\textwidth}
  \centering
  \includegraphics[width=1.0\linewidth,keepaspectratio]{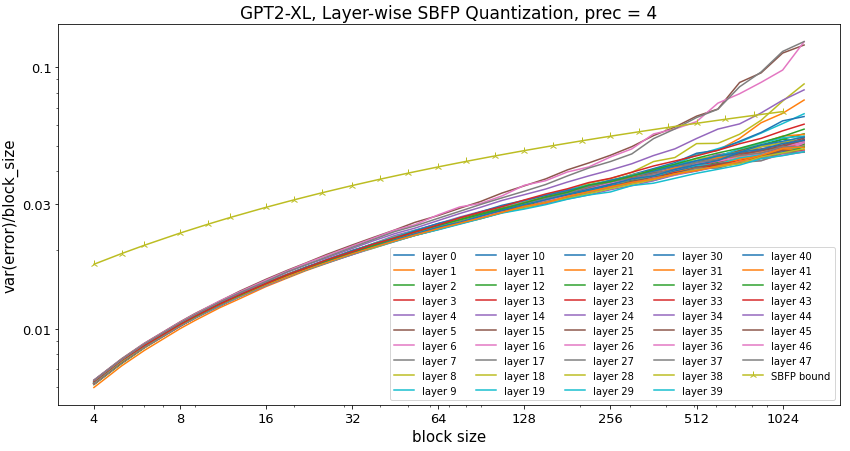}
  \caption{SBFP}
  \label{fig:gpt2_sbfp}
\end{subfigure}%
} \par
\makebox[\textwidth]
{
\begin{subfigure}{.5\textwidth}
  \centering
  \includegraphics[width=1.0\linewidth,keepaspectratio]{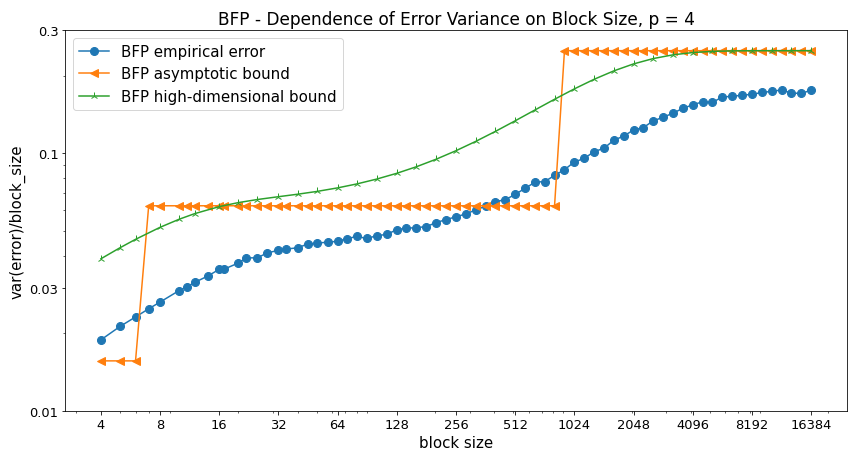}
  \caption{BFP}
  \label{fig:hd_bfp}
\end{subfigure}
\begin{subfigure}{0.5\textwidth}
  \centering
  \includegraphics[width=1.0\linewidth,keepaspectratio]{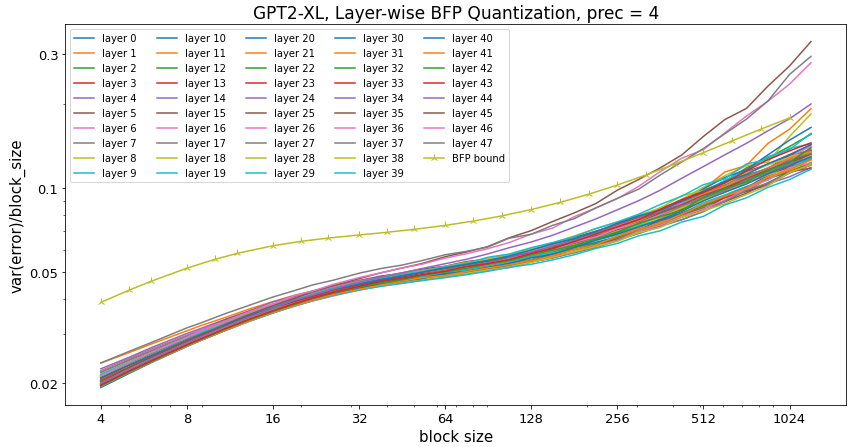}
  \caption{BFP}
  \label{fig:gpt2_bfp}
\end{subfigure}
} \par
\caption{Theoretical error bounds versus the empirical error variances in quantized normal populations and GPT2-XL weights.}
\label{fig:test}
\end{figure*}

Figure \ref{fig:rel_m_theory} shows that the optimal block size grows with the precision and finally saturates (converges) to a finite value around $n=512$. This phenomenon can be explained as follows. Consider the factor of the integrand in equation (\ref{eq:var_bfp_hd_bound}) depending on the block precisions $p_1$ and $p_2$ trough $\alpha_1$ and $\alpha_2$ respectively,
\begin{equation}
2^{2\ceil*{\log_2 \frac{\sigma y_1}{\alpha_1}}} + 2^{2\ceil*{\log_2 \frac{\sigma y_2}{\alpha_2}}}.
\end{equation}
Let us focus on one of the summands. For simplicity of notation consider its logarithm,
\begin{align}
\label{eq:p_mac_exp}
\log_2 & 2^{2\ceil*{\log_2 \frac{\sigma y_j}{\alpha_j}}} = 2\ceil*{\log_2 (\sigma y_j) - \log_2 \alpha_j} \nonumber \\
&= 2\ceil*{\log_2 (\sigma y_j) - \log_2 \[2^{p_j-1}\(1-\frac{1}{2^{p_j-1}}\)\]} \nonumber \\
&= 2\ceil*{\log_2 (\sigma y_j) - (p_j - 1) - \log_2\(1-\frac{1}{2^{p_j-1}}\)} \nonumber \\
&= 2\ceil*{\log_2 (\sigma y_j) + \frac{1}{2^{p_j-1}\ln 2} + o\(\frac{1}{2^{p_j}}\)},
\end{align}
where the last equality follows from the fact that $p_j-1$ is an integer and can be removed from the ceiling function and we used the Maclaurin expansion of the logarithmic function of the form
\begin{equation}
\log_2\(1-x\) = -\frac{x}{\ln 2} + o(x),\;\; x \to 0.
\end{equation}
The last expression in (\ref{eq:p_mac_exp}) shows that for large values of $p_j$, the terms depending on the precision inside the ceiling function vanish. This explains the convergence of the curves in Figure \ref{fig:rel_m_theory} for growing precision to a fixed curve. For our second observation, let us note based on Lemma \ref{lem:gumbel_dist_y} from the Appendix that the values of random variable $y_j$ sharply concentrate around their expected value that grows with $n$ monotonically. Therefore, for smaller values of $n$, the relative added value of $- \log_2\(1-\frac{1}{2^{p_j-1}}\)$ would be more significant when $p_j$ is smaller. This explains the observation that the minima of the REBAC curves are shifting towards smaller block sizes when the precision is reduced.

%Another observation we can make is that for small  conclusion we can make from it as that the additional term depending on $p_j$ is larger for small values of precision which shifts the extremum of the corresponding curve to the left is more significant for smaller values of $y$ which explains why the curves are higher for small values of the block size and shift the minimum to the left.
%
%The dependence of the expected maximum absolute value of a sequence of normal variables $\mathbb{E}\[Y_n\]$ is shown in Figure \ref{}. We can see that it almost saturates at the value of for large block sizes. As a consequence, when the $p$ precision grows larger, the value of the scaling factor $\frac{Y_n}{\alpha}$ almost does not depende on $n$ for large $n$ and only decreases inverse proportional to $\alpha$. As a consequence, starting with some value of $\alpha$, the value of $\log_2\[\frac{Y_n}{\alpha}\]$ becomes very close to zero and is always rounded up to  

\begin{figure*}
\centering
\begin{subfigure}{.5\textwidth}
  \centering
  \includegraphics[width=1.0\linewidth]{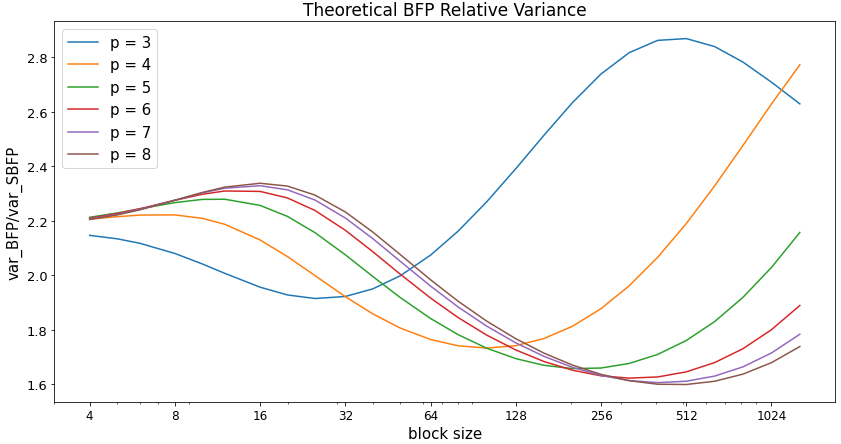}
  \caption{Theoretical behavior of $\rho_{\text{var}}(\text{BFP})$ from (\ref{eq:rho_var}) for Gaussian blocks.}
  \label{fig:rel_m_theory}
\end{subfigure}%
\begin{subfigure}{.5\textwidth}
  \centering
  \includegraphics[width=1.0\linewidth]{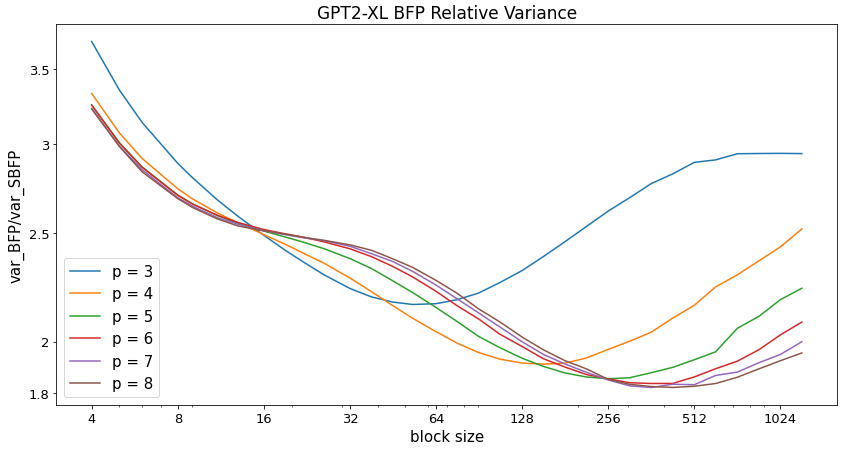}
  \caption{Behavior of $\rho_{\text{var}}(\text{BFP})$ from (\ref{eq:rho_var}) for GPT-XL weights}
  \label{fig:rel_m_gpt2}
\end{subfigure}
\caption{REBAC ratio from equation (\ref{eq:rho_var}) comparing BFP to SBFP as a function of block size for various precisions.}
\end{figure*}

\section{Conclusions}
This article focuses on the rigorous study of block quantization framework that have become one of the main toolboxes enabling modern DNN inference. In particular, we develop asymptotic bounds on the inner product error in SBFP- and BFP-quantized normally distributed vectors and precisely explain the jumps in the values of the latter occurring at the block sizes at which the expected scales pass over the integer powers of two. We then refine those asymptotic bounds to finite dimensional settings and derive high-dimensional tight bounds for the same errors. Here we observe significant differences between the SBFP and BFP related to the aforementioned jumps. In finite-dimentional settings the jumps turn into oscillations around the middle curve, while the SBFP bound is always concave. We next use these observations to introduce a performance measure assessing the typical accuracy of any block format. Such a framework enables us to determine the optimal values of the block format parameters providing the highest possible accuracy. An example of such parameter selection is the block size tuning, which we study in detail. In particular, we show that if the precision of the BFP format is fixed at $4$ bits, the optimal block size becomes $n=64$. All our theoretical derivations are supported by numerical simulations and studies on the weights of publicly available pretrained neural networks.

\section{Appendix}

In the section, we mark the dependence of the distribution of $Y$ on $n$ through the subscript,
\begin{equation}
\label{eq:def_y}
Y_n = \max_{i=1}^n |X_i|.
\end{equation}
Let us also introduce an auxiliary extreme value distribution,
\begin{equation}
\label{eq:def_u}
U_n = \max_{i=1}^n X_i.
\end{equation}
The density of Gumbel distribution with location parameter $\mu$ and scale $\sigma > 0$ reads as
\begin{equation}
\label{eq:gumb_def}
g(x; \mu, \sigma) = \frac{1}{\sigma} e^{-(z+e^{-z})},
\end{equation}
where
\begin{equation}
z = \frac{x-\mu}{\sigma}.
\end{equation}
Let $G$ be Gumbel distributed with parameters $\mu$ and $\sigma$, then its mean and variance read as
\begin{equation}
\mathbb{E}[G] = \mu + \gamma \sigma,
\end{equation}
where $\gamma \approx 0.577$ is the Euler-Mascheroni constant and
\begin{equation}
\var{G} = \frac{\pi^2}{6}\sigma^2.
\end{equation}

\begin{lemma}[Corollary of Fisher-Tippett-Gnedenko Theorem \cite{frechet1927loi, fisher1928limiting, mises1936distribution, gnedenko1943distribution}]
\label{lem:fisher}
Let $X_i \sim \mathcal{N}(0, 1)$ be i.i.d. and $U_n$ be defined by (\ref{eq:def_u}), then its asymptotic distribution is Gumbel with parameters
\begin{equation}
\mu_n^U = \sqrt{\ln\(\frac{n^2}{2\pi\,\ln\(\frac{n^2}{2\pi}\)}\)} + o\(\frac{1}{\sqrt{\ln\, n}}\),\;\; n \to \infty,
\end{equation}
\begin{equation}
\sigma_n^U = \sqrt{\ln\(\frac{n^2}{2\pi\,\ln\(\frac{n^2}{2\pi}\)}\)}\frac{1}{\ln n} + o\(\frac{1}{\sqrt{\ln\, n}}\), n \to \infty.
\end{equation}
\end{lemma}

\begin{lemma}
\label{lem:gumbel_dist_y}
Let $X_i \sim \mathcal{N}(0, 1)$ be i.i.d., then the asymptotic distribution of $Y_n$ is Gumbel with parameters
\begin{equation}
\mu_n = \mu_n^Y = \sqrt{\ln\(\frac{2n^2}{\pi\,\ln\(\frac{2n^2}{\pi}\)}\)} + o\(\frac{1}{\sqrt{\ln\, n}}\),\;\; n \to \infty,
\end{equation}
\begin{equation}
\sigma_n = \sigma_n^Y = \sqrt{\ln\(\frac{2n^2}{\pi\,\ln\(\frac{2n^2}{\pi}\)}\)}\frac{1}{\ln(2n)} + o\(\frac{1}{\sqrt{\ln\, n}}\), n \to \infty.
\end{equation}
Therefore, the mean and variance of $Y_n$ are given by
\begin{align}
\label{eq:mean_y}
\mathbb{E}[Y_n] &= \mu_n + \gamma \sigma_n = \sqrt{\ln\(\frac{2n^2}{\pi\,\ln\(\frac{2n^2}{\pi}\)}\)} \nonumber \\
&\times\(1 + \frac{\gamma}{\ln(2n)} + o\(\frac{1}{\ln\, n}\)\),\;\; n \to \infty,
\end{align}
\begin{multline}
\label{eq:var_y}
\var{Y_n} = \frac{\pi^2}{6}\sigma_n^2 = \frac{\pi^2}{6} \ln\(\frac{2n^2}{\pi\,\ln\(\frac{2n^2}{\pi}\)}\)\frac{1}{\log^2(2n)} \\ + o\(\frac{1}{\ln\, n}\),\;\; n \to \infty.
\end{multline}
\end{lemma}

\begin{proof}
Note that the distribution of $Y_n$ converges to $U_{2n}$ an $n$ grows. Now the claim follows from Lemma \ref{lem:fisher}.
\end{proof}

%\begin{lemma}
%\label{lem:max_norm}
%Let $X_i \sim \mathcal{N}(0, 1)$ be i.i.d. and $Y_n = \max_i |X_i|$, then
%\begin{equation*}
%\mathbb{E}[Y_n] = \cmr{\sqrt{\ln\(\frac{4n^2}{2\pi\,\ln\(\frac{2n^2}{\pi}\)}\)}}\(1 + \frac{\gamma}{\ln(2n)} + o\(\frac{1}{\ln n}\)\),\;\; n \to \infty,
%\end{equation*}
%where $\gamma \approx 0.577$ is the Euler-Mascheroni constant,
%\begin{equation*}
%\var{Y_n} = \frac{\pi^2}{6} \cmr{\ln\(\frac{4n^2}{2\pi\,\ln\(\frac{2n^2}{\pi}\)}\)\frac{1}{\log^2(2n)}} + o\(\frac{1}{\ln n}\),\;\; n \to \infty.
%\end{equation*}
%\end{lemma}

%\begin{definition}[\cite{vershynin2018high}]
%\label{def:subgauss}
%We shall say that a real centered random variable $W$ is sub-Gaussian with variance proxy $\sigma^2>0$ if
%\begin{equation}
%\mathbb{E}\[e^{sW}\] \leqslant e^{\frac{\sigma^2s^2}{2}},\;\; \forall s \in \mathbb{R}.
%\end{equation}
%\end{definition}

%\begin{definition}[\cite{vershynin2018high}]
%\label{def:subgauss}
%We shall say that a real centered random variable $W$ is sub-Gaussian with variance proxy $\sigma^2>0$ if
%\begin{equation}
%\mathbb{E}\[W^{2q}\] \leqslant q!(4\sigma^2)^q,\;\; \forall q \in \mathbb{N}.
%\end{equation}
%\begin{equation}
%\mathbb{E}\[W^{2q}\] \leqslant 2q!(2\sigma^2)^q,\;\; \forall q \in \mathbb{N}.
%\end{equation}
%\end{definition}

\begin{lemma}[Properties of sub-Gaussian random variables, \cite{vershynin2018high, bernstein1924modification}]
\label{lem:subgaus_prop}
If $W$ and $U$ are independent symmetric sub-Gaussian random variables with variance proxies $\sigma_W^2$ and $\sigma_U^2$, respectively, then
\begin{enumerate}
\item $W + U$ is sub-Gaussian with variance proxy $\sigma_W^2 + \sigma_U^2$,
\item the following tail bounds apply to any of them
\begin{equation}
\label{eq:bernst_bound}
\mathbb{P}\[W \geqslant t\] =\mathbb{P}\[W \leqslant -t\] \leqslant \exp\(-\frac{t^2}{2\sigma_W^2}\),\;\; \forall t \geqslant 0.
\end{equation}
\end{enumerate}
\end{lemma}

%\begin{lemma}[\cite{vershynin2018high}]
%If the even moments of a symmetric random variance $W$ satisfy
%\begin{equation}
%\mathbb{E}\[W^{2q}\] \leqslant 2q q!(2\sigma^2)^q,\;\; \forall q \in \mathbb{N},
%\end{equation}
%then it is sub-Gaussian with variance proxy $8\sigma^2$.
%\end{lemma}

%To bound the probability of large deviations of the error we will use the following lemma.
%\begin{lemma}[Bernstein Inequality, \cite{bernstein1924modification}]
%\label{lem:bern}
%Let $W_i$ be independent sub-Gaussian random variables with variance proxies $\sigma_i^2 > 0$, respectively, then
%\begin{equation}
%\mathbb{P}\[\left|\sum_i W_i\right| \geqslant t\] \leqslant 2\exp\(-\frac{t^2}{2\sum_i\sigma_i^2}\),\;\; \forall t \geqslant 0.
%\end{equation}
%\end{lemma}

\begin{lemma}
\label{lem:fubini}
For two independent centered real random variables $X$ and $Y$,
\begin{equation}
\mathbb{E}\[(XY)^q\] = \mathbb{E}\[X^q\]\mathbb{E}\[Y^q\].
\end{equation}
\end{lemma}
\begin{proof}
Follows from Fubini's Theorem.
\end{proof}

\begin{lemma}[$\Delta E_s$ Concentration for fixed $Y^{(1)}$ and $Y^{(2)}$]
\label{lem:fixed_y_bound}
For fixed $Y^{(1)}$ and $Y^{(2)}$, $\Delta E_s$ is a symmetric sub-Gaussian random variable with variance proxy
\begin{align}
\label{eq:sumd_vp}
\sigma_{\Delta E_s\, |\, Y^{(1)}, Y^{(2)}}^2 &= (n-1)\[Y^{(1)}\]^2\frac{\sigma^2}{4\alpha_1^2} + (n-1)\[Y^{(2)}\]^2\frac{\sigma^2}{4\alpha_2^2} \nonumber \\
&\qquad + (n-2)\[Y^{(1)}Y^{(2)}\]^2\frac{1}{2^7 \alpha_1^2\alpha_2^2}.
\end{align}
%\begin{multline}
%\label{eq:fixed_y_bound}
%\mathbb{P}\[\left|\Delta E \right| \geqslant t \Big | Y^{(1)}, Y^{(2)}\] \\ \leqslant 2\exp\(-\frac{t^2}{2(n-1)\[[Y^{(1)}]^2\frac{\sigma^2}{4\alpha_1^2} + [Y^{(2)}]^2\frac{\sigma^2}{4\alpha_2^2} +[Y^{(1)}Y^{(2)}]^2\frac{1}{(12\alpha_1\alpha_2)^2}\]}\).
%\end{multline}
\end{lemma}
\begin{proof}

%\begin{align*}
%\Delta E_s &= \sum_{i}X_i^{(1)}X_i^{(2)} - \frac{Y^{(1)}}{\alpha_1}\frac{Y^{(2)}}{\alpha_2}\sum_i \mathcal{I}\[Z_i^{(1)}\]\mathcal{I}\[Z_i^{(2)}\] \\
%& = \frac{Y^{(1)}}{\alpha_1}\frac{Y^{(2)}}{\alpha_2}\[\sum_i Z_i^{(1)}Z_i^{(2)} - \mathcal{I}\[Z_i^{(1)}\]\mathcal{I}\[Z_i^{(2)}\]\]
%\end{align*}

Recall that
\begin{align}
\Delta E_s & = \frac{Y^{(1)}}{\alpha_1}\frac{Y^{(2)}}{\alpha_2}\[\sum_i Z_i^{(1)}Z_i^{(2)} - \mathcal{I}\[Z_i^{(1)}\]\mathcal{I}\[Z_i^{(2)}\]\] \nonumber \\
&= \frac{Y^{(1)}}{\alpha_1}\frac{Y^{(2)}}{\alpha_2}\[\sum_i \Delta Z_i^{(1)}Z_i^{(2)} + \Delta Z_i^{(2)}Z_i^{(1)} \right. \nonumber \\
& \qquad\qquad \left. + \Delta Z_i^{(1)} \Delta Z_i^{(2)}\].
\end{align}
%Therefore,
%\begin{multline}
%\mathbb{P}\[\frac{Y^{(1)}}{\alpha_1}\frac{Y^{(2)}}{\alpha_2}\left|\sum_i \Delta Z_i^{(1)}Z_i^{(2)} + \Delta Z_i^{(1)}Z_i^{(2)} + \Delta Z_i^{(1)} \Delta Z_i^{(2)} \right| \geqslant t \Bigg | Y^{(1)}, Y^{(2)}\] \\ 
%= \mathbb{P}\[\left|\sum_i \frac{Y^{(1)}}{\alpha_1}\Delta Z_i^{(1)} X_i^{(2)} + \frac{Y^{(2)}}{\alpha_2}\Delta Z_i^{(2)}X_i^{(1)} + \frac{Y^{(1)}}{\alpha_1}\frac{Y^{(2)}}{\alpha_2}\Delta Z_i^{(1)} \Delta Z_i^{(2)} \right| \geqslant t \Bigg | Y^{(1)}, Y^{(2)}\].
%\end{multline}
To simplify notation, introduce auxiliary random variables
\begin{align}
D_i &= \frac{Y^{(1)}}{\alpha_1}\Delta Z_i^{(1)} X_i^{(2)} + \frac{Y^{(2)}}{\alpha_2}\Delta Z_i^{(2)}X_i^{(1)} \nonumber \\
&\qquad + \frac{Y^{(1)}}{\alpha_1}\frac{Y^{(2)}}{\alpha_2}\Delta Z_i^{(1)} \Delta Z_i^{(2)},\;\; i=1,\dots,n.
\end{align}
Recall that the values $Y^{(1)}, Y^{(2)}$ are given and assume without loss of generality that the indices $i_*^{(1)}, i_*^{(2)}$ of the block extreme values are known, then $D_i$ are independent of each other for all $i$. 

Recall that $\Delta Z_i^{(j)}$ is assumed to be uniform in the interval $\[-\frac{1}{2},\frac{1}{2}\]$ and therefore its odd moments are zero and its even moments read as
\begin{equation}
\label{eq:unif_mom}
\mathbb{E}\[\(\Delta Z_i^{(j)}\)^{2q}\] = \frac{1}{(2q+1)2^{2q}},\;\; \forall q \in \mathbb{N}.
\end{equation}
Since $X_i^{(3 -j)}$ is normal centered with variance $\sigma^2$ and independent of $\Delta Z_i^{(j)}$ given $Y^{(1)}$ and $Y^{(2)}$, we can apply Lemma \ref{lem:fubini} to infer
\begin{align}
\label{eq:prod_mom_bound}
&\mathbb{E}\[\(X_i^{(3-j)}\Delta Z_i^{(j)}\)^{2q}\] = \mathbb{E}\[\(X_i^{(3-j)}\)^{2q}\]\mathbb{E}\[\(\Delta Z_i^{(j)}\)^{2q}\] \nonumber \\
&=\sigma^{2q}(2q-1)!!\frac{1}{(2q+1)2^{2q}} = 
\frac{\sigma^{2q}(2q-1)!!}{2^{2q}(2q+1)},\;\; \forall q \in \mathbb{N}.
\end{align}
Now let us show that $X_i^{(3-j)}\Delta Z_i^{(j)}$ is sub-Gaussian with variance proxy $\frac{\sigma^2}{4}$. Indeed,
\begin{align}
\label{eq:subg_dzx}
&\mathbb{E}\[\exp\(sX_i^{(3-j)}\Delta Z_i^{(j)}\)\] = 1 + \sum_{q=1}^\infty \frac{s^{2q}\sigma^{2q}(2q-1)!!}{2^{2q}(2q+1)(2q)!} \nonumber \\
&\leqslant 1 + \sum_{q=1}^\infty \frac{(2q)!!q!}{2q(2q)!}\frac{\(\frac{\sigma^2 s^2}{2^2}\)^q}{q!} = 1 + \sum_{q=1}^\infty \frac{q!q!}{2q(2q)!}\frac{\(\frac{\sigma^2 s^2}{2}\)^q}{q!} \nonumber \\
&= 1 + \sum_{q=1}^\infty \frac{1}{2q {2q \choose q}}\frac{\(\frac{\sigma^2 s^2}{2}\)^q}{q!},
\end{align}
where we used the double factorial identity
\begin{equation}
(2q)!! = 2^q q!,\;\; \forall q \in \mathbb{N}.
\end{equation}
Next we use the central binomial coefficient approximation
\begin{equation}
{2q \choose q} = \frac{(2q)!}{q!q!} = \frac{2^{2q}}{\sqrt{q\pi}}\(1+O\(\frac{1}{q}\)\),\;\; q\to \infty,
\end{equation}
following from the Stirling approximation to get
\begin{multline}
\label{eq:xz_subgaus}
\mathbb{E}\[\exp\(sX_i^{(3-j)}\Delta Z_i^{(j)}\)\] \leqslant 1 + \sum_{q=1}^\infty \frac{\sqrt{q\pi}}{2\sqrt{q} 2^{2q}}\frac{\(\frac{\sigma^2 s^2}{2}\)^q}{q!} \\
\leqslant 1 + \sum_{q=1}^\infty \frac{\(\frac{\sigma^2 s^2}{2^3}\)^q}{q!} = e^{\frac{1}{2}\frac{\sigma^2}{4}s^2}.
\end{multline}
This demonstrates that the variance proxy of $X_i^{(3-j)}\Delta Z_i^{(j)}$ is $\frac{\sigma^2}{4}$. Note in addition that the left-hand side of (\ref{eq:xz_subgaus}) is the moment generating function of $X_i^{(3-j)}\Delta Z_i^{(j)}$, therefore, the variance of the latter is bounded by half the variance proxy,
\begin{equation}
\label{eq:var_bound}
\var{X_i^{(3-j)}\Delta Z_i^{(j)}} \leqslant \frac{\sigma^2}{8}.
\end{equation}
In an analogous manner from Lemma \ref{lem:fubini} and (\ref{eq:unif_mom}) we get
\begin{equation}
\mathbb{E}\[\(\Delta Z_i^{(3-j)}\Delta Z_i^{(j)}\)^{2q}\] = \frac{1}{(2q+1)^22^{4q}},\;\; \forall q \in \mathbb{N},
\end{equation}
%\begin{align}
%\label{eq:prod_z_bound}
%\mathbb{E}\[\(\Delta Z_i^{(3-j)}\Delta Z_i^{(j)}\)^{2q}\] = \frac{1}{(2q+1)^22^{4q}} = \frac{1}{(2q+1)^2q!}\(\frac{1}{2^4}\)^qq! \\
%& \leqslant \frac{1}{(2q+1)^2q!}\(4\frac{1}{2^6}\)^q q!,\;\; \forall q \in \mathbb{N}.
%\end{align}
and therefore,
\begin{align}
\label{eq:subg_dz}
&\mathbb{E}\[\exp\(s\Delta Z_i^{(3-j)}\Delta Z_i^{(j)}\)\] = 1 + \sum_{q=1}^\infty \frac{s^{2q}}{(2q+1)^22^{4q}(2q)!} \nonumber \\
&= 1 + \sum_{q=1}^\infty \frac{q!}{(2q+1)^2(2q)!}\frac{\(\frac{s^2}{2^4}\)^q}{q!} \nonumber \\
&\leqslant 1 + \sum_{q=1}^\infty \frac{\sqrt{2\pi q}\(\frac{q}{e}\)^q}{(2q+1)^2\sqrt{4\pi q}\(\frac{2q}{e}\)^{2q}}\frac{\(\frac{s^2}{2^4}\)^q}{q!} \nonumber \\
&\leqslant 1 + \sum_{q=1}^\infty \frac{1}{(2q+1)^2\sqrt{2}q^q}\frac{\(\frac{s^2e}{2^6}\)^q}{q!} \leqslant 1 + \sum_{q=1}^\infty \frac{1}{(4e)^q}\frac{\(\frac{s^2e}{2^6}\)^q}{q!} \nonumber \\
&= 1 + \sum_{q=1}^\infty \frac{\(\frac{s^2}{2^8}\)^q}{q!} = e^{\frac{1}{2}\frac{s^2}{2^7}},
\end{align}
where we have used the Stirling approximation again. From (\ref{eq:subg_dz}) we conclude that $\Delta Z_i^{(3-j)}\Delta Z_i^{(j)}$ is sub-Gaussian with variance proxy $\frac{1}{2^7}$.

Lemma \ref{lem:subgaus_prop} now implies that $D_i$ are sub-Gaussian with variance proxies
\begin{align}
\sigma_{D_i}^2 = \begin{cases}
\[Y^{(1)}\]^2\frac{\sigma^2}{4\alpha_1^2} + \[Y^{(2)}\]^2\frac{\sigma^2}{4\alpha_2^2} \\ \qquad + \[Y^{(1)}Y^{(2)}\]^2\frac{1}{2^7 \alpha_1^2\alpha_2^2}, &\text{ if } i \notin \{i_*^{(1)}, i_*^{(2)}\}, \\ 
\[Y^{(1)}\]^2\frac{\sigma^2}{4\alpha_1^2}, &\text{ if } i = i_*^{(2)}, \\ 
\[Y^{(2)}\]^2\frac{\sigma^2}{4\alpha_2^2}, &\text{ if } i = i_*^{(1)}.
\end{cases}
\end{align}
The same Lemma \ref{lem:subgaus_prop} applied once again yields,
\begin{align}
\sigma_{\sum_i D_i}^2 &= (n-1)\[Y^{(1)}\]^2\frac{\sigma^2}{4\alpha_1^2} + (n-1)\[Y^{(2)}\]^2\frac{\sigma^2}{4\alpha_2^2} \nonumber \\
&+ (n-2)\[Y^{(1)}Y^{(2)}\]^2\frac{1}{2^7 \alpha_1^2\alpha_2^2},
\end{align}
concluding the proof.

%\begin{align}
%\label{eq:p_bound_1}
%&\mathbb{P}\[\left|\sum_i D_i \right| \geqslant t \Bigg | Y^{(1)}, Y^{(2)}\]  \nonumber \\
%&\leqslant 2\exp\(-\frac{t^2}{2\[(n-1)[Y^{(1)}]^2\frac{\sigma^2}{8\alpha_1^2} + (n-1)[Y^{(2)}]^2\frac{\sigma^2}{8\alpha_2^2} + (n-2)[Y^{(1)}Y^{(2)}]^2\frac{1}{3\cdot 2^7 \alpha_1^2\alpha_2^2}\]}\) \nonumber \\
%&\leqslant 2\exp\(-\frac{t^2}{2(n-1)\[[Y^{(1)}]^2\frac{\sigma^2}{8\alpha_1^2} + [Y^{(2)}]^2\frac{\sigma^2}{8\alpha_2^2} +[Y^{(1)}Y^{(2)}]^2\frac{1}{3\cdot 2^7 \alpha_1^2\alpha_2^2}\]}\),
%\end{align}
%where we recalled that only $n-1$ variables $Z_i^{(i)}$ in every block are random and according to (\ref{eq:z_var}) their variance is equal to $1/12$.
\end{proof}

% The probability (\ref{eq:p_bound_1}) is conditioned upon the realizations of $Y^{(1)}$ and $Y^{(2)}$.

\begin{proof}[Proof of Proposition \ref{prop:sbfp_bound_asymp}]
Lemma \ref{lem:fixed_y_bound} establishes that $\Delta E_s$ is a symmetric sub-Gaussian random variable with variance proxy given by equation (\ref{eq:sumd_vp}) for fixed $Y^{(1)}$ and $Y^{(2)}$. We now want to get the unconditional large deviation probability bound for $\Delta E_s$ claimed by Proposition \ref{prop:sbfp_bound_asymp}. To this end, we apply the law of total expectation to the moment generating function to obtain
\begin{equation}
\label{eq:total_mgf}
\mathbb{E}\[e^{s\Delta E_s}\] = \mathbb{E}_{Y^{(1)}, Y^{(2)}}\[\mathbb{E}\[e^{s\Delta E_s} \;\big|\; Y^{(1)}, Y^{(2)}\]\].
\end{equation}
Note that according to Lemma \ref{lem:gumbel_dist_y}, the variance of the distribution of $Y_n$ vanishes when $n$ grows to infinity, meaning that the distributions of $Y^{(1)}$ and $Y^{(2)}$ converge weakly to the Dirac delta function. Hence, the outer integration in (\ref{eq:total_mgf}) over the densities of $Y^{(1)}$ and $Y^{(2)}$ boils down to plugging their mean values that are equal for a fixed value of $n$. From Lemma \ref{lem:fixed_y_bound} we conclude that $\Delta E_s$ is centered sub-Gaussian with variance proxy
\begin{align}
\sigma_{\Delta E_s}^2 &= (n-1)\mathbb{E}\[Y^{(1)}\]^2\frac{\sigma^2}{4\alpha_1^2} + (n-1)\mathbb{E}\[Y^{(2)}\]^2\frac{\sigma^2}{4\alpha_2^2} \nonumber \\
&\qquad + (n-2)\mathbb{E}\[Y^{(1)}\]^2\mathbb{E}\[Y^{(2)}\]^2\frac{1}{2^7 \alpha_1^2\alpha_2^2} \nonumber \\
&= (n-1)\[\mathbb{E}\[\frac{Y^{(1)}}{\sigma}\]^2\frac{\sigma^4}{4\alpha_1^2} + \mathbb{E}\[\frac{Y^{(2)}}{\sigma}\]^2\frac{\sigma^4}{4\alpha_2^2}\] \nonumber \\
&\qquad + (n-2)\frac{\mathbb{E}\[Y^{(1)}\]^2\mathbb{E}\[Y^{(2)}\]^2}{2^7 \alpha_1^2\alpha_2^2}, \;\; n \to \infty.
\end{align}
Note also that according to Lemma \ref{lem:gumbel_dist_y},
\begin{equation}
\frac{Y^{(j)}}{\sigma} \leqslant \sqrt{2\ln (2n)}.
\end{equation}
with overwhelming probability. Hence, for $1 \ll \ln\, n \ll \max_j 2^{p_j}$ we can use Lemma \ref{lem:gumbel_dist_y} once again to upper bound the last summand in the last equation to obtain
\begin{equation}
\label{eq:sumd_vp_uncond_2}
\sigma_{\Delta E_s}^2 \leqslant \[\frac{\sigma^4}{4\alpha_1^2} + \frac{\sigma^4}{4\alpha_2^2}\] n\,\ln\(\frac{2n^2}{\pi\,\ln\(\frac{2n^2}{\pi}\)}\).
\end{equation}
Note that according to equation (\ref{eq:var_bound}) and the explanation around it, the variance of a sub-Gaussian variable is bounded by half its variance proxy to conclude the desired claim.
\end{proof}

\begin{lemma}[$\Delta E_b$ Concentration for fixed $Y^{(1)}$ and $Y^{(2)}$]
\label{lem:fixed_y_bound_bfp}
For fixed $Y^{(1)}$ and $Y^{(2)}$, $\Delta E_b$ is a symmetric sub-Gaussian random variable with variance proxy
\begin{align}
\label{eq:sumd_vp_bfp}
\sigma_{\Delta E_b\, |\, Y^{(1)}, Y^{(2)}}^2 &= \frac{(n-1)\sigma^2}{4}\[2^{2\ceil*{\log_2 \frac{Y^{(1)}}{\alpha_1}}} + 2^{2\ceil*{\log_2 \frac{Y^{(2)}}{\alpha_2}}}\] \nonumber \\
&\quad+ \frac{(n-2)}{2^7}2^{2\ceil*{\log_2 \frac{Y^{(1)}}{\alpha_1}} + 2\ceil*{\log_2 \frac{Y^{(2)}}{\alpha_2}}}.
\end{align}
\end{lemma}
\begin{proof}
The proof follows that of Proposition \ref{lem:fixed_y_bound} verbatim after replacing the block scales
\begin{equation}
\frac{Y^{(j)}}{\alpha_j} \mapsto 2^{\ceil*{\log_2 \frac{Y^{(j)}}{\alpha_j}}},\;\; j=1,2.
\end{equation}
according to the explanation in Section \ref{sec:bfp_quant}.
\end{proof}

\begin{proof}[Proof of Proposition \ref{prop:bfp_bound_asymp}]
The proof follows that of Proposition \ref{prop:sbfp_bound_asymp} verbatim after substituting Lemma \ref{lem:fixed_y_bound} by Lemma \ref{lem:fixed_y_bound_bfp}.
\end{proof}

\begin{proof}[Proof of Proposition \ref{prop:hd_bfp_bound}]
Here we again follow the steps of the proof of Proposition \ref{prop:sbfp_bound_asymp} up to the application of the law of total expectation
\begin{equation}
\label{eq:total_mgf_bfp}
\mathbb{E}\[e^{s\Delta E_b}\] = \mathbb{E}_{Y^{(1)}, Y^{(2)}}\[\mathbb{E}\[e^{s\Delta E_b} \;\big|\; Y^{(1)}, Y^{(2)}\]\].
\end{equation}
Note again that with overwhelming probability
\begin{equation}
\frac{\alpha_j}{\sigma}2^{\ceil*{\log_2 \frac{Y^{(j)}}{\alpha_j}}} \leqslant \sqrt{2\ln (2n)},
\end{equation}
and for $1 \ll \ln\, n \ll \max_j 2^{p_j}$,
\begin{multline}
\label{eq:sumd_vp_bfp_bound}
\sigma_{\Delta E_b\, |\, Y^{(1)} = \sigma y_1, Y^{(2)} = \sigma y_2}^2(y_1,y_2) \\ \leqslant \frac{(n-1)\sigma^2}{4}\[2^{2\ceil*{\log_2 \frac{\sigma y_1}{\alpha_1}}} + 2^{2\ceil*{\log_2 \frac{\sigma y_2}{\alpha_2}}}\].
\end{multline}
For the moment generating function we obtain
\begin{multline}
\mathbb{E}\[e^{s\Delta E_b}\] \leqslant \int \exp\[\frac{\[\sigma_{\Delta E_b\, |\, Y^{(1)}, Y^{(2)}}(y_1,y_2)\]^2s^2}{2}\] \\ \times f_n(y_1)f_n(y_2)dy_1dy_2.
\end{multline}
where $f_n(\cdot)$ is the density defined in equation (\ref{eq:density_phi}). As a consequence, for the second moment we get
\begin{align}
&\var{\Delta E_b} \nonumber \\
&\leqslant \frac{1}{2} \int \sigma_{\Delta E_b\, |\, Y^{(1)} = \sigma y_1, Y^{(2)} = \sigma y_2}^2(y_1,y_2) f_n(y_1)f_n(y_2)dy_1dy_2 \nonumber \\
&\leqslant \frac{1}{2} \int \frac{(n-1)\sigma^2}{4}\[2^{2\ceil*{\log_2 \frac{\sigma y_1}{\alpha_1}}} + 2^{2\ceil*{\log_2 \frac{\sigma y_2}{\alpha_2}}}\] \nonumber \\
&\qquad\qquad \times f_n(y_1)f_n(y_2)dy_1dy_2,
\end{align}
which concludes the proof.
\end{proof}

\begin{proof}[Proof of Proposition \ref{prop:hd_sbfp_bound}]
The proof follows that of Proposition \ref{prop:hd_bfp_bound} verbatim after replacing $2^{\ceil*{\log_2 \frac{Y^{(j)}}{\alpha_j}}}$ with $\frac{Y^{(j)}}{\alpha_j}$ in (\ref{eq:sumd_vp_bfp_bound}).
\end{proof}

\bibliographystyle{IEEEtran}
\bibliography{ilya_bib}

\end{document}

%% file: ilyas_commands.tex
\renewcommand{\(}{\left(}
\renewcommand{\)}{\right)}
\renewcommand{\[}{\left[}
\renewcommand{\]}{\right]}

\newcommand{\End}[1]{{\rm{End}}}

\renewcommand{\log}[1]{{\rm{log}}#1}

\newcommand{\var}[1]{{\rm{var}}\[#1\]}

\newtheorem{lemma}{Lemma}
\newtheorem{definition}{Definition}

\newtheorem{prop}{Proposition}